\documentclass[onecolumn]{article} 
\usepackage[american]{babel}

\usepackage{natbib} 
\usepackage{mathtools} 
\usepackage{booktabs} 
\usepackage{tikz} 

\usepackage[margin=2.8cm]{geometry}
\usepackage{times}

\usepackage{soul}
\usepackage{url}
\usepackage[utf8]{inputenc}
\usepackage{graphicx}
\usepackage{amsmath,amssymb}
\usepackage{mleftright}
\usepackage{braket}
\usepackage{natbib}
\usepackage{authblk}

\newtheorem{theorem}{Theorem}
\newtheorem{lemma}{Lemma}
\newtheorem{definition}{Definition}
\newenvironment{proof}{\underline{Proof.}\\}{$_\square$}

\usepackage{algorithm}
\usepackage{algorithmic}

\usepackage{cleveref}
\usepackage{subfig}

\newcommand{\bp}{{\bf p}}
\newcommand{\bb}{{\bf b}}
\newcommand{\bt}{{\bf t}}
\newcommand{\bd}{{\bf d}}
\newcommand{\bm}{{\bf m}}
\newcommand{\bx}{{\bf x}}
\newcommand{\by}{{\bf y}}
\newcommand{\bs}{{\bf s}}
\newcommand{\bI}{{\bf I}}

\newcommand{\p}{{\bp_{\bd}} }
\newcommand{\Simplex}[1]{{\Delta^{{#1}-1}}}
\newcommand{\bez}{{\bf b}_{\{\p\}} }
\newcommand{\R}{\mathbb{R}}
\newcommand{\N}{\mathbb{N}}
\newcommand{\E}{\mathbb{E}}
\newcommand{\supp}{\text{supp}}
\newcommand{\ball}[2]{B_{#1}^{#2}}
\newcommand{\vol}[2]{|B_{#1}^{#2}|}
\newcommand{\Coeff}[2]{\begin{pmatrix}
#1 \\
#2
\end{pmatrix}}

\newcommand{\sqbra}[1]{\mleft[#1\mright]}

\title{Approximate Bayesian Computation of B\'ezier Simplices}

\author[1,2,3]{Akinori Tanaka\footnote{e-mail: \texttt{\{akinori.tanaka, akiyoshi.sannai\}@riken.jp }}}
\author[2,3]{Akiyoshi Sannai${}^{\ast}$}
\author[2,4,5]{Ken Kobayashi\footnote{e-mail: \texttt{ken-kobayashi@fujitsu.com}}}
\author[2,6]{Naoki Hamada\footnote{e-mail: \texttt{hamada-n@klab.com}}}

\affil[1]{RIKEN Interdisciplinary Theoretical and Mathematical Sciences Program}
\affil[2]{RIKEN Center for Advanced Intelligence Project}
\affil[3]{Keio University}
\affil[4]{Fujitsu Limited}
\affil[5]{Tokyo Institute of Technology}
\affil[6]{KLab Inc.}

\date{}

\begin{document}
\maketitle
\begin{abstract}
B\'ezier simplex fitting algorithms have been recently proposed to approximate the Pareto set/front of multi-objective continuous optimization problems.
These new methods have shown to be successful at approximating various shapes of Pareto sets/fronts when sample points exactly lie on the Pareto set/front.
However, if the sample points scatter away from the Pareto set/front, those methods often likely suffer from over-fitting.
To overcome this issue, in this paper, we extend the B\'ezier simplex model to a probabilistic one and propose a new learning algorithm of it, which falls into the framework of approximate Bayesian computation (ABC) based on the Wasserstein distance.
We also study the convergence property of the Wasserstein ABC algorithm.
An extensive experimental evaluation on publicly available problem instances shows that the new algorithm converges on a finite sample.
Moreover, it outperforms the deterministic fitting methods on noisy instances.
\end{abstract}

\section{Introduction}
Multi-objective optimization is a ubiquitous task in our life, which is the problem of minimizing multiple objective functions under certain constraints, denoted by
\begin{align*}
    \text{minimize } &~f(x) := (f_1(x), \dots, f_M(x)) \\
    \text{subject to } &~x \in X (\subseteq \R^N).
\end{align*}
Since the objective functions $f_1,\dots,f_M: X\to \R$ are usually conflicting, we would like to consider their minimization according to the Pareto order defined as follows:
\[
    x \prec y \xLeftrightarrow{\mathrm{def}} \forall i \sqbra{f_i(x) \leq f_i(y)} \land \exists j \sqbra{f_j(x) < f_j(y)}.
\]
The goal is to find the \emph{Pareto set}
\[
    X^*(f) := \Set{x \in X | \forall y \in X \sqbra{y \not\prec x}}
\]
and the \emph{Pareto front}
\[
    f(X^*(f)) := \Set{f(x) \in \R^M | x \in X^*(f)}.
\]

By solving the problem with numerical methods (e.g., goal programming \citep{Miettinen1999,Eichfelder2008}, evolutionary computation \citep{Deb2001,Zhang2007,Deb2014}, homotopy methods \citep{Hillermeier2001,Harada2007}, Bayesian optimization \citep{Hernandez-Lobato2016,Yang2019}), we obtain a finite set of points as an approximation of the Pareto set/front.
In almost all cases, the dimensionality of the Pareto set and front is $M-1$ (see \citep{Wan1977,Wan1978} for rigorous statement), which means that such a finite-point approximation suffers from the ``curse of dimensionality''.
As the number of objective functions $M$ increases, the number of points required to represent the entire Pareto set and front grows exponentially.

Fortunately, it is known that many problems have a simple structure of the Pareto set and front, which can be utilized to enhance approximation.
\cite{Kobayashi2019} defined the \emph{simplicial} problem whose Pareto set/front is homeomorphic to an $(M - 1)$-dimensional simplex and each $(m - 1)$-dimensional subsimplex corresponds to the Pareto set/front of an $m$-objective subproblem for all $1 \le m \leq M$ (see \Cref{fig:face-relation}).
\cite{Hamada2020} showed that strongly convex problems are simplicial under mild conditions.
They also showed that facility location \citep{Kuhn1967} and phenotypic divergence modeling in evolutionary biology \citep{Shoval2012} are simplicial.
In airplane design \citep{Mastroddi2013} and hydrologic modeling \citep{Vrugt2003}, scatter plots of numerical solutions imply those problems are simplicial.
\cite{Kobayashi2019} showed that the Pareto set and front of any simplicial problem can be approximated with arbitrary accuracy by a B\'ezier simplex.
\begin{figure*}[t]
    \centering
    \includegraphics[width=\linewidth]{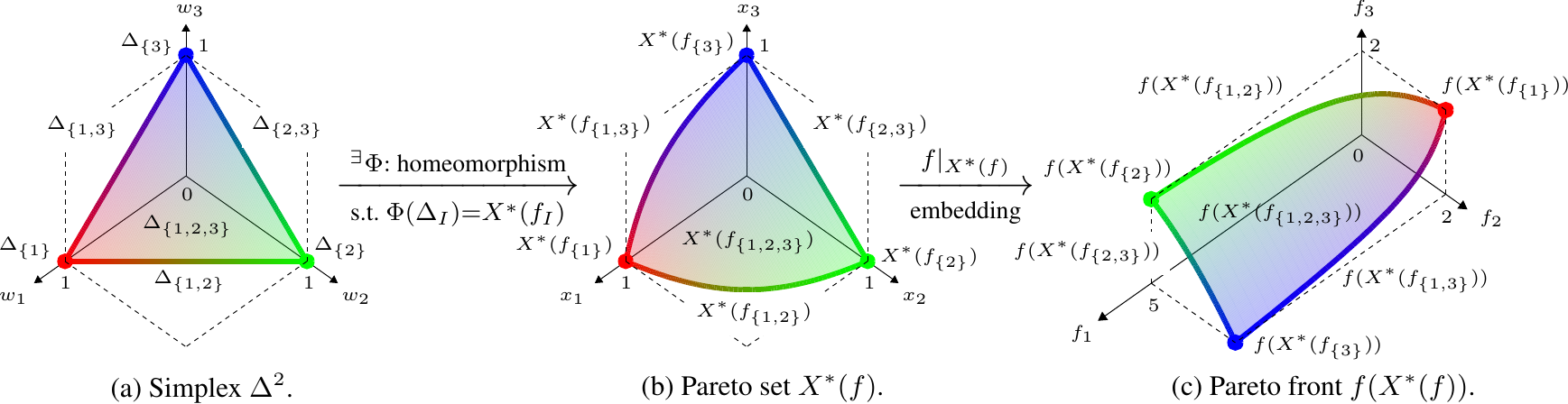}
    \caption{A simplicial problem $f = (f_1, f_2, f_3): \R^3 \to \R^3$. An $M$-objective problem $f$ is simplicial if the following conditions are satisfied: (i) there exists a homeomorphism $\Phi: \Delta^{M - 1} \to X^*(f)$ such that $\Phi(\Delta_I) = X^*(f_I)$ for all $I \subseteq \set{1, \dots, M}$; (ii) the restriction $f|_{X^*(f)}: X^*(f) \to \R^M$ is a topological embedding (and thus so is $f \circ \Phi: \Delta^{M - 1} \to \R^M$).}\label{fig:face-relation}
\end{figure*}

A serious drawback of B\'ezier simplex fitting methods is the lack of robustness against noise.
\cite{Kobayashi2019} assumed that the sample points lie exactly on the Pareto set/front. \cite{Tanaka2020} showed that introducing noise to the sample points drastically degrades the quality of approximation.
Since multi-objective optimization often requires a great amount of computing resources, the outcome of numerical optimization methods may be rough, which is considered as a noisy sample from the Pareto set/front.

To make the fitting methods noise-resilient, in this paper, we extend the B\'ezier simplex model to a Bayesian one, propose an approximate Bayesian computation based on the Wasserstein distance, which we call WABC, and evaluate its asymptotic behavior.
Our contributions are as follows:
\begin{itemize}
    \item In~\Cref{sec:bezier-simplex-fitting}, we propose an approximate Bayesian computation algorithm for B\'ezier simplex fitting (\Cref{alg1}), which is an extension of the deterministic fitting algorithm proposed in \citep{Kobayashi2019}.
    \item In~\Cref{sec:wabc}, we analytically evaluate the bias of the WABC posterior (\Cref{th:3}), and the acceptance rate of WABC (\Cref{th:4}).
    \item In~\Cref{sec:experiments}, we demonstrate the usefulness of the proposed method on noisy data sets.
\end{itemize}

\section{B\'ezier Simplex Fitting}\label{sec:bezier-simplex-fitting}
In this section, we review the definition of the B\'ezier simplex and some known regression algorithms.
The existing methods, however, are basically deterministic algorithms and vulnerable to noise.
To solve this problem, we propose a new Bayesian inference-based method.

\subsection{B\'ezier simplex}
Let us begin with the definition of the $(M-1)$-dimensional simplex,
\[
    \Simplex{M} 
    = 
    \Set{
        \bt = (t_1, t_2, \dots, t_M) \in \R^M 
        | 
        0 \leq t_m, \sum_{m=1}^M t_m = 1  
    }
    .
\]
See \Cref{fig:face-relation}(a), which is $(M-1)=2$ case for example.

A B\'ezier simplex of order $D$ is defined as a polynomial map of $D$-degree from $\Simplex{M}$ to $\R^M$.
To label each monomial, we introduce the set of \emph{degrees}:
\[
    \N_D^M 
    = 
    \Set{
        \bd = (d_1, d_2, \dots, d_M) \in \N^M 
        | 
        \sum_{m=1}^M d_m = D  
    }
    .
\]
In addition, we need to assign a vector $\p \in \R^M$, which we call \emph{control point}, to each degree $\bd \in \N_D^M$.
For a given set of control points $\{ \p \in \R^M \}_{\bd \in \N_D^M}$, that we write it $\{ \p \}$ for short, the B\'ezier simplex is defined as
\[
    \bez(\bt) 
    = 
    \sum_{\bd \in \N_D^M} 
        \Coeff{D}{\bd} 
        t_1^{d_1} t_2^{d_2} \cdots t_M^{d_M}
        \p
    ,
\]
where $\bt \in \Simplex{M}$, and $\Coeff{D}{\bd}$ is the multinomial coefficient.

\subsection{Regression Model}
From the definition, it is evident that B\'ezier simplices are continuous functions from $\Simplex{M}$ to the target space $\R^M$.
Due to such a topological property, fitting the given data by a B\'ezier simplex works well if the data enjoyed the underlying simplex structure.
The goal of the fitting is finding a set of control points $\{ \p \}$ that the corresponding B\'ezier simplex $\bez$ reproduces data points, and there is some known prior work based on regression as follows.

\paragraph{The all-at-once fitting algorithm.}
\cite{Kobayashi2019} proposed a B\'ezier simplex fitting algorithm: the all-at-once fitting.
The all-at-once fitting requires a training set $\set{\bx_i \in \R^M | i = 1, \dots, n}$ and aim to adjust all control points $\{\p\}$ by minimizing the OLS loss: $\frac{1}{n}\sum_{i = 1}^n \|\bx_i - \bez(\bt_i)\|^2$, where $\bt_i~(i=1,\dots,n)$ is a parameter corresponding to $\bx_i$. 
Here, the parameter $\bt_i$ for each $\bx_i$ is unknown in advance. 
Thus, the all-at-once fitting repeats minimizing the OLS loss with respect to the parameters $\bt_i~(i=1,\dots,n)$ and the control points $\{\p\}$ alternatively.

\paragraph{The inductive-skeleton fitting algorithm.}
\cite{Kobayashi2019} also propose a sophisticated method utilizing the simplex structure of the B\'ezier simplex: the inductive skeleton fitting.
In this method, the authors recursively applied the all-at-once fitting to subsimplices of the B\'ezier simplex.

\paragraph{Drawbacks.}
Both methods easily become unstable when fitting noisy samples.
When their fitting algorithms minimize the OLS with respect to $\bt_i~(i=1,\dots,n)$, the solution $\bt_i$ is a foot of a perpendicular line from $\bx_i$ to $\bez(\Delta^{M-1})$.
This calculation requires to solve nonlinear equations with Newton's method, which is quite sensitive to noises included in $\bx_i$ and also sometimes fails to converge.

\subsection{Extension to Bayesian Model}

To resolve the above problems, we propose a new fitting algorithm based on Bayesian inference.
First of all, we describe the B\'ezier simplex model in a probabilistic manner.

Let $U_\Simplex{M}(\bt)$ be the uniform distribution on the $(M-1)$-dimensional simplex $\Simplex{M}$.
We define the likelihood of a point $\bx$ as
\begin{align}
p(\bx|\{\p\}) 
= \int_{\Simplex{M}} 
  \delta_{\bx - \bez(\bt)} 
  U_{\Simplex{M}}(\bt) 
  d\bt,
\label{model}
\end{align}
where $\delta_\bx$ is the Dirac distribution with mass on $\bx$.
Note that this conditional probability can be regarded as defining a generative model, which is the push-forward of the uniform distribution $U_\Simplex{M}(\bt)$ by the B\'ezier simplex, i.e.
\begin{align}
\by \sim p(\cdot|\{\p\}) \Leftrightarrow \bt \sim U_\Simplex{M}(\cdot), \by = \bez(\bt).
\label{sampling}
\end{align}

\paragraph{Bayesian inference.}
To describe our Bayesian treatment of B\'ezier simplices, we focus on the general framework of Bayesian inference with parameter $\theta$ for the time being\footnote{We can recover the B\'ezier simplex model by taking $\theta = \{ \p \}$ and the likelihood \eqref{model}.}.
In addition, we introduce a shorthand notation for the vector made of concatenating some vectors sharing the same dimension.

\begin{definition}[aligned vector]
\label{def:ali}
Suppose $\{ \bx_i \}_{i=1,2,\dots,n}$ is a set of $M$-dimensional vectors, we define the $(nM)$-dimensional vector,
\begin{align*}
    \bx_{1:n}
    &:= 
    \begin{array}{ll}
        &[(\bx_1)_1, (\bx_1)_2, \dots, (\bx_1)_M, \\
        &\ (\bx_2)_1, (\bx_2)_2, \dots, (\bx_2)_M, \\
        &\ \dots\\
        &\ (\bx_n)_1, (\bx_n)_2, \dots, (\bx_n)_M],
    \end{array}
    =:
    [\bx_1: \bx_2: \cdots : \bx_n]
    ,
\end{align*}
by concatenating $\{ \bx_i \}_{i=1,2,\dots,n}$ in order of indexing, where $(\bx_i)_\mu$ is the $\mu$-th component of the vector $\bx_i$.
\end{definition}


We use this notation to represent given data and synthetic data generated by models.
For example, we denote the likelihood function of given data $\{ \bx_i \}_{i=1,2,\dots,n}$ as
\begin{align}
    p(\bx_{1:n}|\theta) 
    = 
    \prod_{i=1}^n p(\bx_i|\theta)
    . \label{prod}
\end{align}
One of the main targets of Bayesian inference is the posterior distribution for $\theta$,
\begin{align}
p_{\mathrm{posterior}}(\theta | \bx_{1:n}) 
=  \frac{p(\bx_{1:n}|\theta) p_{\mathrm{prior}}(\theta)}{p(\bx_{1:n})},
\label{posterior}
\end{align}
where $p_{\mathrm{prior}}(\theta)$ is a certain prior distribution, and $p(\bx_{1:n})$ is the marginal distribution.
Once we got the explicit form of the posterior, we can directly treat the model in probabilistic way based on it. 
However, this is not always possible.
In fact, we cannot write down the explicit form of our likelihood function \eqref{model}, neither the posterior in that case.



\paragraph{Approximate Bayesian Computation.}
So we need an alternative way.
In the present paper, we use the approximate Bayesian computation (ABC), which is an approximated sampling method from the posterior, defined in \Cref{alg0}.
To execute the algorithm, we only need to know how to sample from the model distribution.

\begin{algorithm}[t]            
\caption{$\{\theta_i \}_{i=1,2,\dots,N_\mathrm{ABC}} \sim p_\mathrm{ABC}^{(\delta)}(\theta|\bx_{1:n})$}         
\label{alg0}                          
\begin{algorithmic}                  
\REQUIRE threshold $\delta$, number of samples $N_\mathrm{ABC}$, data $\bx_{1:n}$, distance $d$
\STATE $\theta_\mathrm{accepted} = \phi$
\WHILE{length$(\theta_\mathrm{accepted}) < N_\mathrm{ABC}$}
\STATE draw parameter from the prior: $\theta \sim p_{\mathrm{prior}}(\theta)$
\STATE draw $\by_{1:m} = [\by_1: \by_2: \dots : \by_m ]$ for $\by_j \sim p(\by|\theta)$ (i.i.d.)
\IF{$d(\bx_{1:n}, \by_{1:m}) \leq \delta$}
\STATE $\theta_\mathrm{accepted} = \theta_\mathrm{accepted}\cup\{\theta\}$
\ENDIF
\ENDWHILE
\RETURN{$\theta_\mathrm{accepted}$}
\end{algorithmic}
\end{algorithm}

We can formally write down the probability distribution for $\theta$ that corresponds to the above sampling algorithm as
\begin{align}
p_{\mathrm{ABC}}^{(\delta)}(\theta | \bx_{1:n}) 
= \frac{\int_{d(\bx_{1:n}, \by_{1:m}) \leq \delta} p(\by_{1:n}|\theta) p_{\mathrm{prior}}(\theta) d\by_{1:m}}{p_\mathrm{ABC}^{(\delta)}(\bx_{1:n})},
\label{ABCposterior}
\end{align}
where $d(\bx_{1:n}, \by_{1:m})$ is a certain distance function that characterizes the algorithm itself, and $p_\mathrm{ABC}^{(\delta)}(\bx_{1:n})$ is the normalization factor.

The choice of the distance function strongly influences the performance of the algorithm.
In the literature \citep{barber2015rate, 10.1007/978-3-319-33507-0_7}, it is common to introduce a certain $q$-dimensional summary statistics $\bs_{\bx} = T(\bx_{1:n})$, $\bs_{\by} = T(\by_{1:m})$,  and define the distance function as $d(\bx_{1:n}, \by_{1:m}) := d_\mathrm{E} (\bs_{\bx}, \bs_{\by})$ by using the Euclidean distance:
\begin{align}
d_\mathrm{E}({\bf a}, {\bf b}) = \sqrt{\sum_{\mu=1}^q |({\bf a})_\mu - ({\bf b})_\mu|^2}.
\label{euc}
\end{align}

In such a case, there is a known result on the bias of the expected value for an arbitrary function of $\theta$ calculated by ABC.

\begin{theorem}[\cite{barber2015rate}]
\label{th:1}
Let $h(\theta)$ be a function of $\theta$ that $\E_{\mathrm{posterior}}[h(\theta)]$ is not divergent.
If the summary statistics $T$ is sufficient and the likelihood is three times continuously differentiable with respect to $\bs_\bx = T(\bx_{1:n})$, then the following equality holds for ABC defined by the Euclidean distance:
\[
    \E_\mathrm{ABC}[h(\theta)] 
    = 
    \E_{\mathrm{posterior}}[h(\theta)] 
    + C_h(\bs_\bx) \delta^2
    + \mathcal{O}(\delta^3)
    ,
\]
where $C_h(\bs_\bx)$ is a value depending only on $\bs_\bx$ and the function $h$.
\end{theorem}

The bias term always scales as $\delta^2$, and it means that the ABC sampling based on the statistics $T$ and the Euclidean distance is unbiased by taking limit $\delta \to 0$.
It sounds good, however, there is a complementary theorem on the acceptance of the ABC procedure.
To illustrate it, we introduce the definition of the Euclidean ball:

\begin{definition}[Euclidean Ball]
\label{def:ball}
We define the $q$-dimensional Euclidean ball with radius $\delta$ and center $\bx$ as follows,
\begin{align*}
    B_{\delta}^q(\bx) 
    = 
    \Set{
        \by \in \R^q
        |
        d_\mathrm{E} (\bx, \by) \leq \delta
    },  
\end{align*}
where $d_\mathrm{E}(\cdot, \cdot)$ is the Euclidean distance defined in \eqref{euc}.
In addition to it, we denote its $\bx$-independent volume as $|\ball{\delta}{q}(\bx)| = \vol{\delta}{q}$.
\end{definition}

Here, we know the following explicit form:
\begin{align}
    \vol{\delta}{q} 
    = 
    \frac{\pi^{q/2}}{\Gamma(\frac{q}{2} + 1)}
    \delta^q
    ,
    \label{volume}
\end{align}
where $\Gamma$ is the Gamma function.
So it scales as $\delta^q$, and this fact is useful to draw out the meaning of the following theorem on the acceptance probability of the ABC procedure.

\begin{theorem}[\cite{barber2015rate}]
\label{th:2}
If the summary statistics $T$ is sufficient and the likelihood is three times continuously differentiable with respect to $\bs_\bx = T(\bx_{1:n})$, then the acceptance probability of ABC defined by the Euclidean distance is 
\[
    p_\mathrm{accept} = p(\bs_\bx)
    \vol{\delta}{q} ( 1 + o(1) )
    ,
\]
and we need to run $N_\mathrm{ABC}/p_\mathrm{accept}$ accept/reject trials during \Cref{alg0} in average for gathering $N_\mathrm{ABC}$ samples.
\end{theorem}

Combining to the scale $\delta^q$ in the volume $\vol{\delta}{q}$, it means that the computational complexity explodes when $\delta \to 0$.
One remedy for it is to adjust the prior $p_{\mathrm{prior}}(\theta)$ so that $p(\bs_\bx) = \int p(\bs_\bx|\theta)p_{\mathrm{prior}}(\theta) d\theta$ takes relatively large value, and that is our strategy for Bayesian  B\'ezier  simplex  model.

\paragraph{ABC of B\'ezier simplices.}
Now, let us come back to the B\'ezier  simplex  model: $\theta = \{ \p \}$, and likelihood defined in \eqref{model} tentatively.
On the prior $p_{\mathrm{prior}}(\{ \p \})$, we take it as follows:
\begin{align}
p_{\mathrm{prior}}(\{ \p \}) 
= \prod_{\bd \in \N_D^M} p(\p|\bm_\bd, \Sigma_\bd),
\label{prior}
\end{align}

where $p(\bp|\bm, \Sigma)$ is the multivariate normal distribution with hyper parameters, mean vector $\bm$ and covariance matrix $\Sigma$.

If the likelihood were also defined by a certain multivariate normal distribution, then the prior \eqref{prior} is conjugate prior.
In such a case, Bayesian updates of the posterior can be reduced to updates of hyperparameters, $\{ \bm_\bd, \Sigma_\bd \}$.

It is not clear whether our likelihood \eqref{model} defined by B\'ezier simplices is conjugate to the prior \eqref{prior}, however, we apply the updates of $\{ \bm_\bd, \Sigma_\bd \}$ calculated by ABC samples.
It may sound out of theory a little, but we would like to emphasize that:

\begin{enumerate}
    \item If there is just one peak, the posterior is well approximated by the multivariate normal distribution around the peak.
    \item We can discard the description of the ``posterior updates,'' and consider successive Bayesian inference just by changing its prior.
\end{enumerate}

\begin{algorithm}[t]            
\caption{ABC of B\'ezier simplex}         
\label{alg1}                          
\begin{algorithmic}                  
\REQUIRE data $\bx_{1:n}$, degree of B\'ezier simplex $D$, positive integers $N_\mathrm{updates}, N_\mathrm{ABC}, N_\delta$, initial $\{\bm_\bd, \Sigma_\bd \}$ and $\delta>0$
\FOR{$n_\mathrm{trial}$ in range($N_\mathrm{updates}$)}
\STATE set the prior as $p_{\mathrm{prior}}(\{ \p \}) = p(\{ \p \} | \{ \bm_\bd, \Sigma_\bd \})$
\STATE draw control points: $\{ \p \}_{1:N_\mathrm{ABC}} \sim p_{\mathrm{ABC}}^{(\delta)}(\{ \p \} | \bx_{1:n}) $
\STATE $\{ \bm_\bd, \Sigma_\bd \} \leftarrow \{\text{mean}( (\p)_{1:N_\mathrm{ABC}})$, $\text{cov}((\p)_{1:N_\mathrm{ABC}}) \}$
\STATE update $\delta$ as follows
\STATE \quad calculate "mean $d$ between model and data"
\STATE \quad\quad $\{ \p \}_{1:N_\delta} \sim p(\{ \p \} | \{ \bm_\bd, \Sigma_\bd \})$
\STATE \quad\quad $(\by_{1:n})_{1:N_\delta} \sim p(\cdot|\{ \p \}_{1:N_\delta})$
\STATE \quad\quad $\epsilon = \frac{1}{N_\delta} \sum_{\alpha=1}^{N_\delta} d(\bx_{1:n}, (\by_{1:n})_\alpha)$
\STATE \quad update $\delta$
\STATE \quad\quad$\delta = 0.9 \epsilon$
\ENDFOR
\RETURN{$\{\bm_\bd, \Sigma_\bd \}$}
\end{algorithmic}
\end{algorithm}

In addition, it would be better to take smaller $\delta$ because of \Cref{th:1}.
To do so, we apply an adaptive update of $\delta$ also. 
We describe our algorithm for the B\'ezier simplex model in \Cref{alg1}.

\section{Wasserstein ABC and its Convergence Properties}\label{sec:wabc}

The next issue on applying ABC to B\'ezier simplex is the choice of the summary statistics $T$ and the distance function $d$ in defining ABC posterior \eqref{ABCposterior}.
Our likelihood \eqref{model} is slightly complicated, and it seems to be difficult to get any good (dimensionally reduced) sufficient statistics as \Cref{th:1} and \Cref{th:2} assumed.

There is, however, a trivial sufficient statistics, $T = \mathrm{id}$.
In this case, we can employ the usual the Euclidean distance to measure two point clouds: $d(\bx_{1:n}, \by_{1:n}) = d_\mathrm{E}(\bx_{1:n}, \by_{1:n})$ by concatenating $\bx_i$ and $\by_i$ into $q=nM$-dimensional vectors as defined in \Cref{def:ali}.
In this case, we can guarantee the convergence and estimate its bias by utilizing \Cref{th:1}.

However, such implementation is not practical because of too low acceptance probability.
If we apply \Cref{th:2} by substituting $q=nM$, then, the denominator scales $(nM)!!$ approximately, and it means that the almost samples will be rejected if the number of samples $n$ increased.

So we take another \emph{better} choice using \emph{Wasserstein distance}, and we call the resultant ABC algorithm as WABC algorithm.
In the later experiments, we use the 2-Wasserstein distance to measure the distance between the data $\bx_{1:n}$ and the model samples $\by_{1:m}$.
In addition, we take $m=n$ for simplicity.
In such case, the definition of the 2-Wasserstein distance reduces to the following form:
\begin{align}
d_\mathrm{W}(\bx_{1:n}, \by_{1:n}) = \sqrt{\min_{\sigma \in S_n} \frac{1}{n}\sum_{i=1}^n d_\mathrm{E}(\bx_i, \by_{\sigma(i)})^2},
\label{Wass2}
\end{align}
where $\sigma \in S_n$ runs for permutation of $n$ indices.

In this case, there is also a known work proving its convergence at $\delta \to 0$ \citep{bernton2019approximate}.
In practical use, however, it would be important to know how the expected value based on WABC with nonzero $\delta$ differs from the expected value for the true posterior.

We can prove that the bias between WABC posterior and the true posterior also scales as $\delta^2$ like \Cref{th:1}.
This is one of the main theorems of the present paper:

\begin{theorem}
\label{th:3}
Let $h(\theta)$ be a function of $\theta$ that $\E_{\mathrm{posterior}}[h(\theta)]$ is not divergent, and the likelihood is three times continuously differentiable with respect to $\bx_{1:n}$, then WABC also has order $\delta^2$ bias:
\[
    \E_\mathrm{WABC}[h(\theta)] 
    = 
    \E_{\mathrm{posterior}}[h(\theta)] 
    + C_h(\bx_{1:n}) \delta^2
    + \mathcal{O}(\delta^3),
\]
where $C_h(\bx_{1:n})$ is a value depending only on $\bx_{1:n}$ and the function $h$.
\end{theorem}

We give its proof in the Appendix.
Here, instead, we show some empirical results with simpler models here. 
In the first case (a), we define our Bayesian model as
\[
    p_{\mathrm{prior}} = N(0, 1),
    \quad
    p(\cdot|\theta) = N(\theta, 1)
    ,
\]
and one-dimensional data $x_{1:n}$ generated by
$
p_\mathrm{data} = N(-1.5, 1)
$.
As well known, prior $N(0, 1)$ is a conjugate prior for $N(\theta, 1)$, and we can easily get the corresponding posterior, which gives 
\[
    \E_{\mathrm{posterior}}[\theta] 
    =
    \frac{\sum_{i=1}^n x_i}{n+1}.
\]

In the next setup (b), we define our Bayesian model as
\[
    p_{\mathrm{prior}} = N(0, 1),
    \quad
    p(\cdot|\theta) 
    = 
    \begin{cases}
        U(0, \theta)     & (\theta > 0)  \\
        U(\theta, 0)     & (\theta < 0) 
    \end{cases}
    ,
\]
and data $x_{1:n}$ generated by
$
p_\mathrm{data} = U(0, 1)
,
$
where $U(a, b)$ is the uniform distribution.
In this case, the posterior is slightly non-trivial, but we can calculate the expectation of $\theta$ as
\[
    \E_{\mathrm{posterior}}[\theta] 
    = \frac{
        \int_1^\infty 
        \frac{
            e^{- \frac{\theta^2}{2}}
            }{
            |\theta|^{n-1}
            }
        d\theta
        }{
        \int_1^\infty 
        \frac{
            e^{- \frac{\tilde{\theta}^2}{2}}
            }{
            |\tilde{\theta}|^{n}
            }
        d\tilde{\theta}
        }
    ,
\]
which we assume $\max \{ x_i \} = 1$ for simplicity.


\begin{figure}[t]
    \centering
    \subfloat[Case 1: Gaussian]{\includegraphics[width=0.5\hsize]{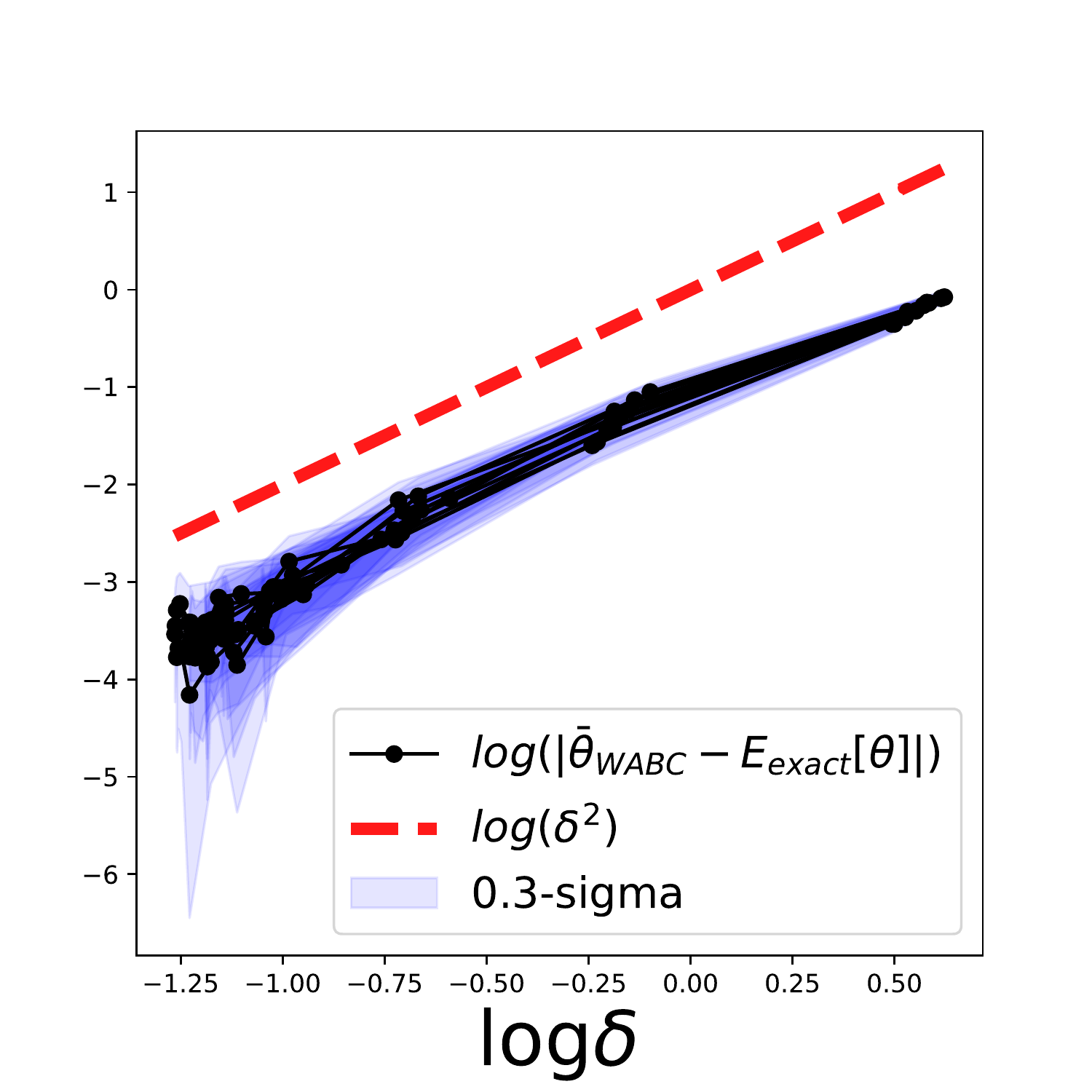}}
    \subfloat[Case2: Uniform]{\includegraphics[width=0.5\hsize]{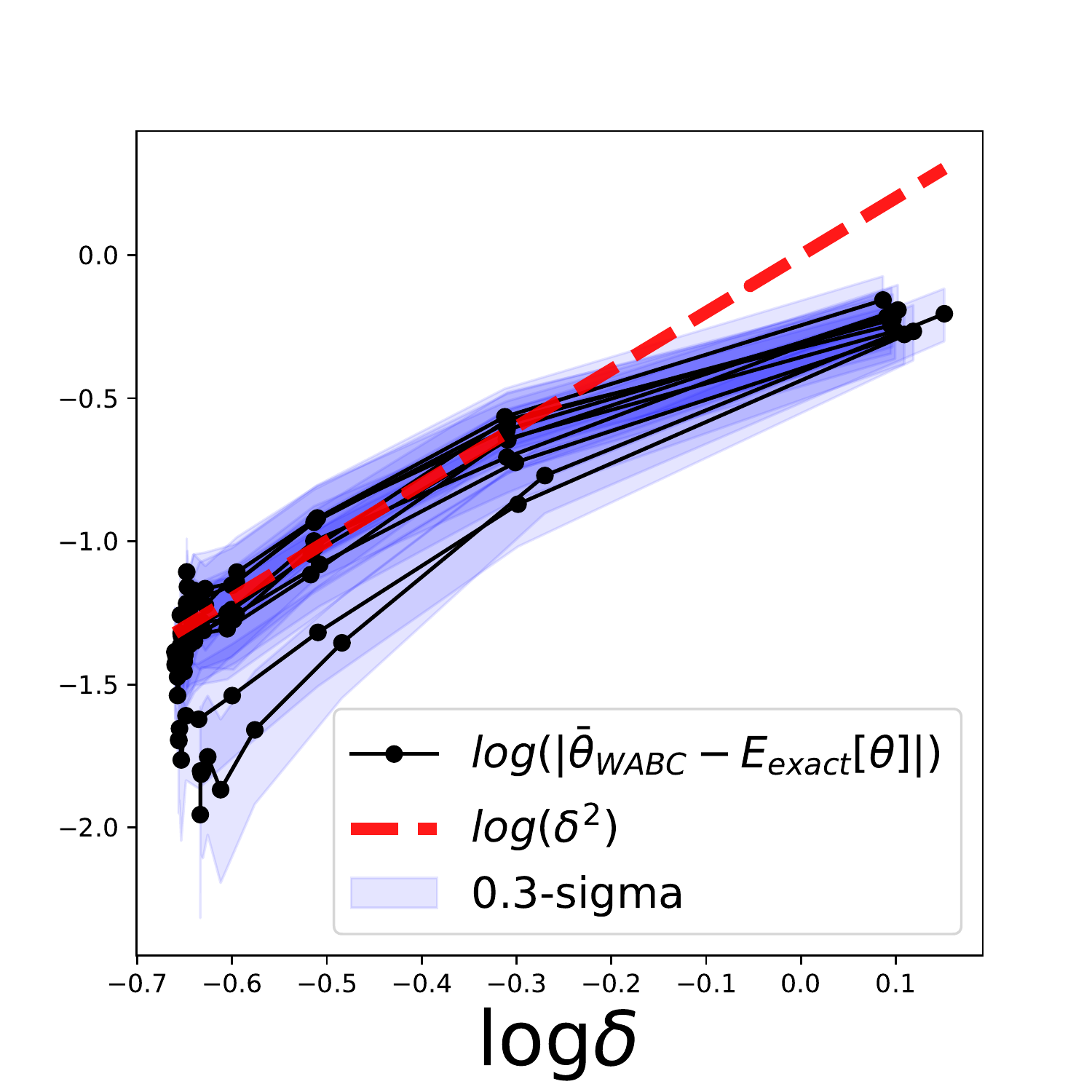}}
    \caption{
    Numerical experiments to check \Cref{th:3}. Horizontal axis is $log \delta$, vertical axis is $\log|\bar{\theta}_\mathrm{WABC} - \E_{\mathrm{posterior}}[\theta] |$ for black/blue plot, and $\log \delta^2$ (prediction by \Cref{th:3}) for red line. 
    Their coefficients seem to be identical near $\delta\approx0$.
    }
    \label{fig:toys}
\end{figure}

In \Cref{fig:toys}, we plot $\log|\bar{\theta}_\mathrm{WABC} - \E_{\mathrm{posterior}}[\theta] |$ on the vertical axis and $\log \delta$ on the horizontal axis with 10 trials with datasize $n=100$.
$\bar{\theta}_\mathrm{WABC}$ is the empirical value of $\E_\mathrm{WABC}[\theta]$ defined as
\[
    \bar{\theta}_\mathrm{WABC} = \frac{1}{N_\mathrm{ABC}} \sum_{i=1}^{N_\mathrm{ABC}} (\theta_\mathrm{accepted})_i.
\]
We compute it with $N_\mathrm{ABC} = 1000$ by \Cref{alg0} based on the Wasserstein distance \eqref{Wass2}.

In addition, we perform a linear regression $\log|\bar{\theta}_\mathrm{WABC} - \E_{\mathrm{posterior}}[\theta] |\approx c_1\log \delta + c_0$ for each trial and calculate the mean and std of the coefficient $c_1$:
\begin{align*}
    \begin{array}{c|cc}
    & \text{$c_1$(fitted by all points)} & \text{$c_1$(fitted by middle points)} \\
    \hline
    \text{(a)} & 1.98 \pm 0.05 & 2.13 \pm 0.16 \\
    \text{(b)} & 1.65 \pm 0.21 & 2.08 \pm 0.33
    \end{array}
\end{align*}
All results show $c_1 \approx 2$ and support \Cref{th:3}, i.e. the $\delta^2$ scaling in both cases\footnote{
Technically, these experiments are sensitive to precision around $\delta \approx 0$.
For example, the errorbar in \Cref{fig:toys}(a) around $\log \delta \approx -1.00$ grows larger than that of around $\log \delta \approx 0.50$.
}.
As the final theoretical result, let us show a scale of acceptance probability in WABC.

\begin{theorem}
\label{th:4}
If the likelihood is three times continuously differentiable with respect to $\bx_{1:n}$, then the acceptance probability of WABC is 
\[
    p_\mathrm{accept} = (n!) p(\bx_{1:n})
    \vol{\sqrt{n}\delta}{q} 
    ( 1 + o(1) )
\]
where $q=nM$.
We need to run $N_\mathrm{ABC}/p_\mathrm{accept}$ accept/reject trials during \Cref{alg0} in average for gathering $N_\mathrm{ABC}$ samples.
\end{theorem}

We give the proof of \Cref{th:4} in the Appendix.
Compared to \Cref{th:2} by using $T=\mathrm{id}$, the acceptance probability is increased by $(n!)$ scale in numerator that cancels $1/(nM)!!$ in $\vol{\sqrt{n}\delta}{q}$, and it means the acceptance is drastically improved.

\paragraph{Remark}

We restrict ourselves within the 2-Wasserstein distance that makes proof a little simple, however, the scaling laws with respect to $\delta$ in \Cref{th:3} and \Cref{th:4} seem to be valid even if we take the $p$-Wasserstein distance to define WABC.



\section{Experiments}\label{sec:experiments}
Now, we turn to the ABC for B\'ezier simplex fitting.
In this section, we examine performances of our proposed method based on the 2-Wasserstein distance for B\'ezier simplex fitting using Pareto front samples. 

\subsection{Experimental settings}
\paragraph{Data Sets}
In the numerical experiments, we employed four synthetic problems with known Pareto fronts, all of which are simplicial problems. 
Schaffer is a two-objective problem, whose Pareto front is a curved line that can be triangulated into two vertices and one edge. 
3-MED and Viennet2 are three objective problems. 
Their Pareto fronts are a curved triangle that can be triangulated into three vertices, three edges and one face. 
5-MED is a five-objective problem. 
Its Pareto front is a curved pentachoron. 
We generated Pareto front samples of 3-MED and 5-MED by AWA(objective) using default hyper-parameters~\citep{Hamada2010}. 
Pareto front samples of the other problems were taken from jMetal 5.2. 
The formulation and data description of each problem are shown in the Appendix.

\paragraph{Settings}
To assess the robustness of the proposed method, we consider the fitting problem with training points including the Gaussian noise. 
For a given dataset, we randomly choose $n$  points $\{\bx_i\}_{i=1,2,\ldots,n}$ as a training dataset.  
Then, we consider fitting a B\'ezier simplex $\bb(\bt)~(\bt \in \Delta^{M-1})$ to the set $\{\tilde{\bx}_i\}_{i=1,2,\ldots,n}$, where $\tilde{\bx}_i = \bx_i + \varepsilon_i~(\varepsilon_i \sim N(\boldsymbol{0}, \sigma \bI)).$
Also, we use the set $\{\bx_i\}_{i=1,2,\ldots,n}$ as a validation dataset not including the Gaussian noise.
For each dataset, experiments were conducted on all combinations of the sample size $n\in\{50, 100, 150\}$ and the noise scale $\sigma\in \{0, 0.05, 0.1\}$.

In this experiment, we estimated the B\'ezier simplex with degree $D=3$, and compared the following two methods:
\begin{description}
\item[All-at-once] the all-at-once fitting \citep{Kobayashi2019},
\item[WABC] our proposed ABC algorithm~(\Cref{alg1}) based on the 2-Wasserstein distance.
\end{description}
Here, we omitted the inductive skeleton fitting \citep{Kobayashi2019} from our comparison because it requires a stratified subsample from skeletons of a simplex and such subsample is not included in each dataset we used. 
For WABC, we set the maximum number of update $N_{\mathrm{update}}=50$ and $N_\delta=100$.
In addition, we set the initial $\bm_{\bd}$ as the vertices of the Be\'zier simplex, to be the single objective optima, and the rest of the control points were set to be the simplex grid spanned by them. 
The initial $\Sigma_{\bd}$ was set to $0.1\times\bI$. 
With the initial $\{\bm_{\bd}, \Sigma_{\bd}\}$, we set initial $\delta$ as the mean of the Wasserstein distances:
\begin{equation*}
    \delta = \frac{1}{N_\delta}\sum_{\alpha=1}^{N_\delta}d_\mathrm{W}(\tilde{\bx}_{1:n},(\by_{1:n})_\alpha), 
\end{equation*}
where $\{\p\}_{1:N_\delta}\sim p(\{\p\}|\{\bm_{\bd},\Sigma_{\bd}\})$ and $(\by)_{1:N_\delta}\sim p(\cdot|\{\p\}_{1:N_\delta)}$.
We terminated WABC~(\Cref{alg1}) when the number of accepted $\theta=\{\p\}$ did not reach $N_\delta$ after it generated $10^5$ sets of control points, or when the updated $\Sigma_{\bd}$ satisfied the following condition:
\begin{equation*}
    \max_{\bd \in \N_{D}^M}\lambda_{\max}(\Sigma_{\bd}) \leq 10^{-5},
\end{equation*}
where $\lambda_{\max}(\Sigma_{\bd})$ is the maximum eigenvalue of $\Sigma_{\bd}$.

\paragraph{Performance Measures}
To evaluate the approximation quality of estimated hyper-surfaces, we used the generational distance (GD)~\citep{Veldhuizen99} and the inverted generational distance (IGD)~\citep{Zitzler03}:
\begin{align*}
    \textrm{GD}(X,Y) \coloneqq \frac{1}{|X|}\sum_{x\in X} \min_{y\in Y}\|x-y\|,\\
    \textrm{IGD}(X,Y) \coloneqq \frac{1}{|Y|}\sum_{y\in Y} \min_{x\in X}\|x-y\|,
\end{align*}
where $X$ is a finite set of points sampled from an estimated hyper-surface and Y is a validation set. 
Roughly speaking, GD and IGD assess precision and recall of the estimated B\'ezier simplex, respectively. 
That is, when the estimated front has a false positive area that is far from the Pareto front, GD gets high. Conversely, when the Pareto front has a false negative area that is not covered by the estimated front, IGD gets high. 
Thus, we can say that the estimated hyper-surface is close to the Pareto front if and only if both GD and IGD are small.
When we calculate GD and IGD for the all-at-once fitting and WABC, we randomly generated $1000$ parameters $\{\bt_i\}_{i=1,2,\ldots,1000}$ from the uniform distribution $U_\Simplex{M}(\cdot)$ and set $X = \{\bb(\bt_i)\}_{i=1,2,\ldots,1000}$ as a set of sampled points on the estimated B\'ezier simplex.
For each dataset and $(n,\sigma)$, we repeated experiments 20 times with different training sets and computed the average and the standard deviations of their GDs and IGDs.

All methods were implemented in Python 3.7.1 and run on Ubuntu 18.04 with an Intel Xeon CPU E5-2640 v4 (2.40GHz) and 128 GB RAM. 
When we calculated the Wasserstein distance, we solved a linear optimization problem with \texttt{POT} package~\citep{flamary2017pot}. 

\subsection{Results}
\Cref{tab:result_N100} shows the average and the standard deviation of the computation time, GD and IGD when $n=100$ and $\sigma\in \{0,0.05,0.1\}$.
The results with $n=50$ and 150 are provided in the supplementary materials (see \Cref{tab:result_N50} and \Cref{tab:result_N150}).
In \Cref{tab:result_N100}, we highlighted the best score of GD and IGD out of the all-at-once fitting and WABC where the difference is at a significant with significant level $p=0.05$ by the Wilcoxon rank-sum test.

First, let us focus on the results of $M=2$, Schaffer. 
In the case of Schaffer, we can observe that while the all-at-once fitting achieved lower GD and IGD than that of WABC when $\sigma=0$, WABC outperformed the all-at-once fitting in GD when $\sigma=0.1$.
Next, we focus on the results of $M\geq3$. 
From the results of $M\geq 3$ shown in \Cref{tab:result_N100}, we can see that WABC consistently outperformed the all-at-once fitting in both GD and IGD. 
Especially in 5-MED with $\sigma=0.1$, the average GD and IGD of the all-at-once fitting 0.449 and 0.474, respectively, each of which is more than three times larger than that of WABC.  
For computational time, \Cref{tab:result_N100} shows that WABC is generally slower than the all-at-once fitting especially with large $M$. 
In the case of 5-MED with $\sigma=0.0$ for example, the average computational time of WABC is about 3000 seconds, which is more than 100 times longer than that of the all-at-once method. 
For the results with $n=50$ and 150, we can observe the similar tendencies as described above (see \Cref{tab:result_N50,tab:result_N150}).
\begin{table*}[ht]
\footnotesize
    \centering
    \caption{GD and IGD (avg.$\pm$s.d. over 20 trials) with $n=100$. The best scores with significance level $p<0.05$ are shown in bold.}
    \label{tab:result_N100}
    \begin{tabular}{lrrccrcc}
    \toprule
    Problem &$\sigma$ & &WABC &&&All-at-once &  \\
                &         & Time & GD &IGD& Time & GD &IGD\\ \midrule
Schaffer	&0	&53.4	&8.92E-03$\pm$2.66E-03	&9.39E-03$\pm$3.06E-03	&2.8	&2.52E-03$\pm$3.37E-05	&1.29E-03$\pm$3.77E-04\\
($M=2$)		&0.05	&220.1	&1.12E-02$\pm$3.20E-03	&1.09E-02$\pm$3.50E-03	&2.0	&1.25E-02$\pm$4.37E-03	&1.70E-02$\pm$8.67E-03\\
		&0.1	&443.6	&\textbf{1.82E-02$\pm$7.56E-03}	&\textbf{1.84E-02$\pm$9.20E-03}	&2.4	&2.79E-02$\pm$1.22E-02	&2.28E-02$\pm$9.03E-03\\ \hline
3-MED	&0	&660.1	&\textbf{5.02E-02$\pm$2.10E-03}	&\textbf{3.36E-02$\pm$2.39E-03}	&3.4	&1.05E-01$\pm$9.28E-03	&4.21E-02$\pm$2.22E-03\\
($M=3$)		&0.05	&1130.7	&\textbf{5.71E-02$\pm$3.01E-03}	&\textbf{3.93E-02$\pm$4.03E-03}	&3.0	&1.09E-01$\pm$1.10E-02	&5.06E-02$\pm$8.49E-03\\
		&0.1	&805.1	&\textbf{7.99E-02$\pm$6.05E-03}	&\textbf{5.24E-02$\pm$6.45E-03}	&3.9	&1.29E-01$\pm$2.09E-02	&6.95E-02$\pm$2.15E-02\\ \hline
Viennet2	&0	&386.3	&\textbf{2.09E-02$\pm$4.43E-03}	&\textbf{2.57E-02$\pm$3.18E-03}	&10.2	&9.29E+00$\pm$1.29E+01	&9.28E-02$\pm$4.64E-02\\
($M=3$)		&0.05	&1196.7	&\textbf{4.47E-02$\pm$4.24E-03}	&\textbf{3.39E-02$\pm$5.11E-03}	&4.2	&1.04E-01$\pm$5.76E-02	&6.81E-02$\pm$2.16E-02\\
		&0.1	&1113.9	&\textbf{9.50E-02$\pm$1.38E-02}	&\textbf{6.00E-02$\pm$1.46E-02}	&6.3	&1.11E-01$\pm$1.83E-02	&1.14E-01$\pm$3.42E-02\\ \hline
5-MED	&0	&3083.0	&\textbf{8.82E-02$\pm$6.77E-03}	&\textbf{1.28E-01$\pm$1.31E-02}	&29.0	&1.53E-01$\pm$2.42E-02	&1.95E-01$\pm$1.23E-02\\
($M=5$)		&0.05	&2292.1	&\textbf{9.83E-02$\pm$7.70E-03}	&\textbf{1.36E-01$\pm$1.15E-02}	&31.8	&1.82E-01$\pm$4.59E-02	&2.02E-01$\pm$2.23E-02\\
		&0.1	&1846.0	&\textbf{1.24E-01$\pm$9.70E-03}	&\textbf{1.55E-01$\pm$1.05E-02}	&79.3	&4.49E-01$\pm$7.30E-01	&4.74E-01$\pm$8.50E-01\\
		\bottomrule
    \end{tabular}
\end{table*}

\subsection{Discussion}
In this section, we discuss the results of our numerical experiments and the practicality of our WABC. 

\paragraph{Approximation accuracy}
For each problem instance $M\geq3$, WABC consistently achieved better GD and IGD. 
While WABC adjusts control points with the Bayesian estimation, the all-at-once fitting is a point estimation.
Thus, the number of control points to be estimated increases for large $M$ and the all-at-once fitting seems to be over-fitted to the training data when the number of training points is small. 
This is the reason why IGDs given the all-at-once method were worse than those of our WABC with large $M$.

Next, we discuss the results of GD. 
\Cref{fig:sigma0} and \Cref{fig:sigma0.1} depict points on the training dataset, the estimated B\'ezier simplex, and the Pareto-front from the 3-MED with $\sigma=0$ and 0.1, respectively. 
From these figures, we first observe that the B\'ezier simplex estimated by the all-at-once fitting was overly spreading and the one obtained by WABC was not. 
The all-at-once fitting only considers the distance from each training point to the B\'ezier simplex, which does no impose that all points on the B\'ezier simplex are close to the training points.
On the other hand, WABC minimizes the Wasserstein distance between the set of data points and the one on the B\'ezier simplex. 
Thus, WABC can avoid the issue of over-spreading of the estimated B\'ezier simplex, which is the reason why WABC achieved better GDs.

\paragraph{Robustness}
Next, we discuss the robustness of WABC and the all-at-once fitting. 
As we described in the results part, WABC obtained better GDs and IGDs than the all-at-once fitting even when $\sigma$ is large.
Also, comparing \Cref{fig:sigma0} with \Cref{fig:sigma0.1}, we can see that the shape of the B\'ezier simplex estimated by the all-at-once fitting was changed when the training points include noise. 
On the other hand, the shape of the B\'ezier simplex obtained by WABC with $\sigma=0, 0.1$ is not so different from each other. 

The all-at-once fitting calculates the distance between each training point and the B\'ezier simplex by drawing perpendicular lines from the data points to the simplex, which is quite sensitive to the perturbation of training points. 
Whereas, WABC does not require such calculation because it minimizes the Wasserstein distance, which contributes to its robustness against noise.

\begin{figure}[ht]
    \centering
    \subfloat[WABC]{\includegraphics[width=0.5\hsize]{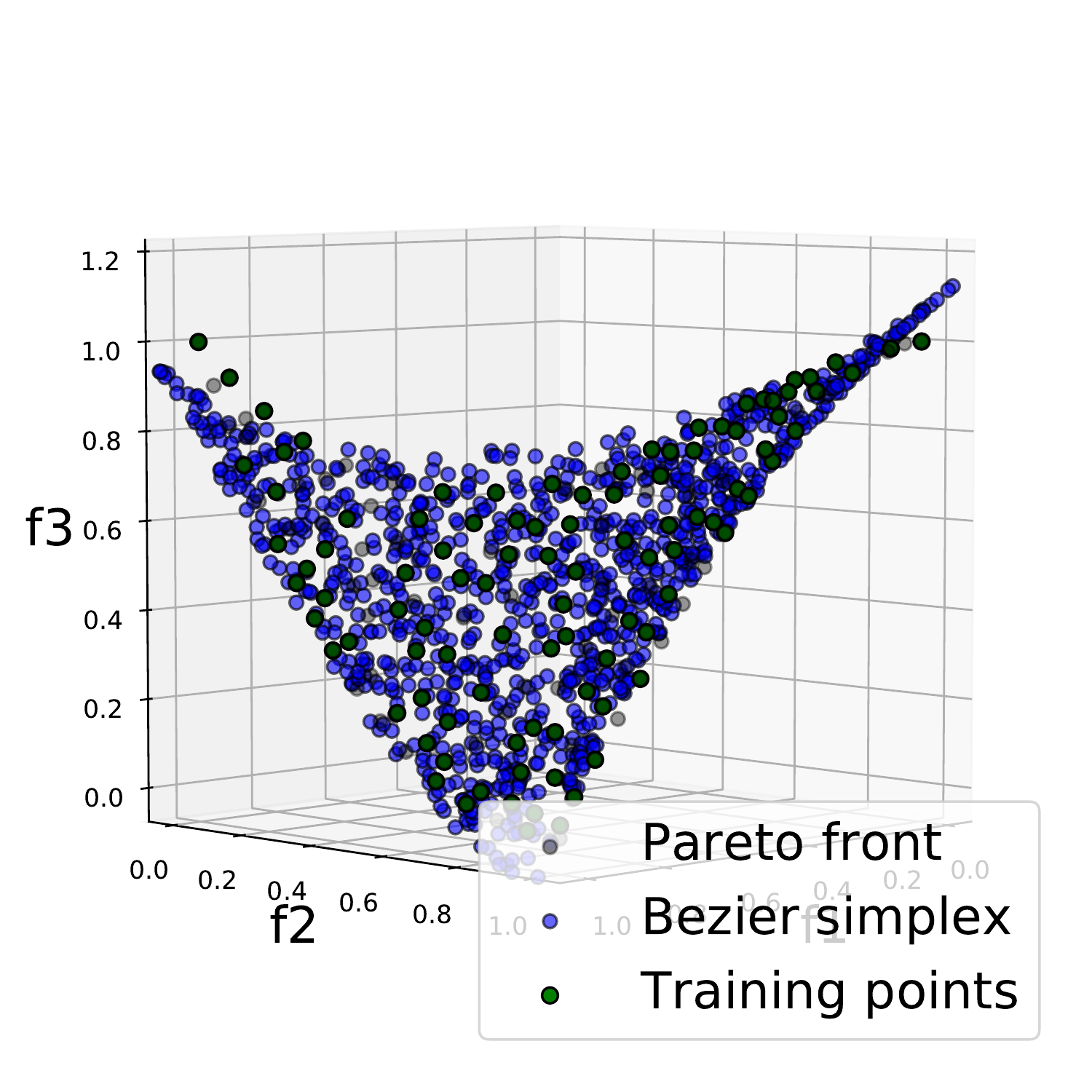}\label{fig:sigma0_abc}}
    \subfloat[All-at-once]{\includegraphics[width=0.5\hsize]{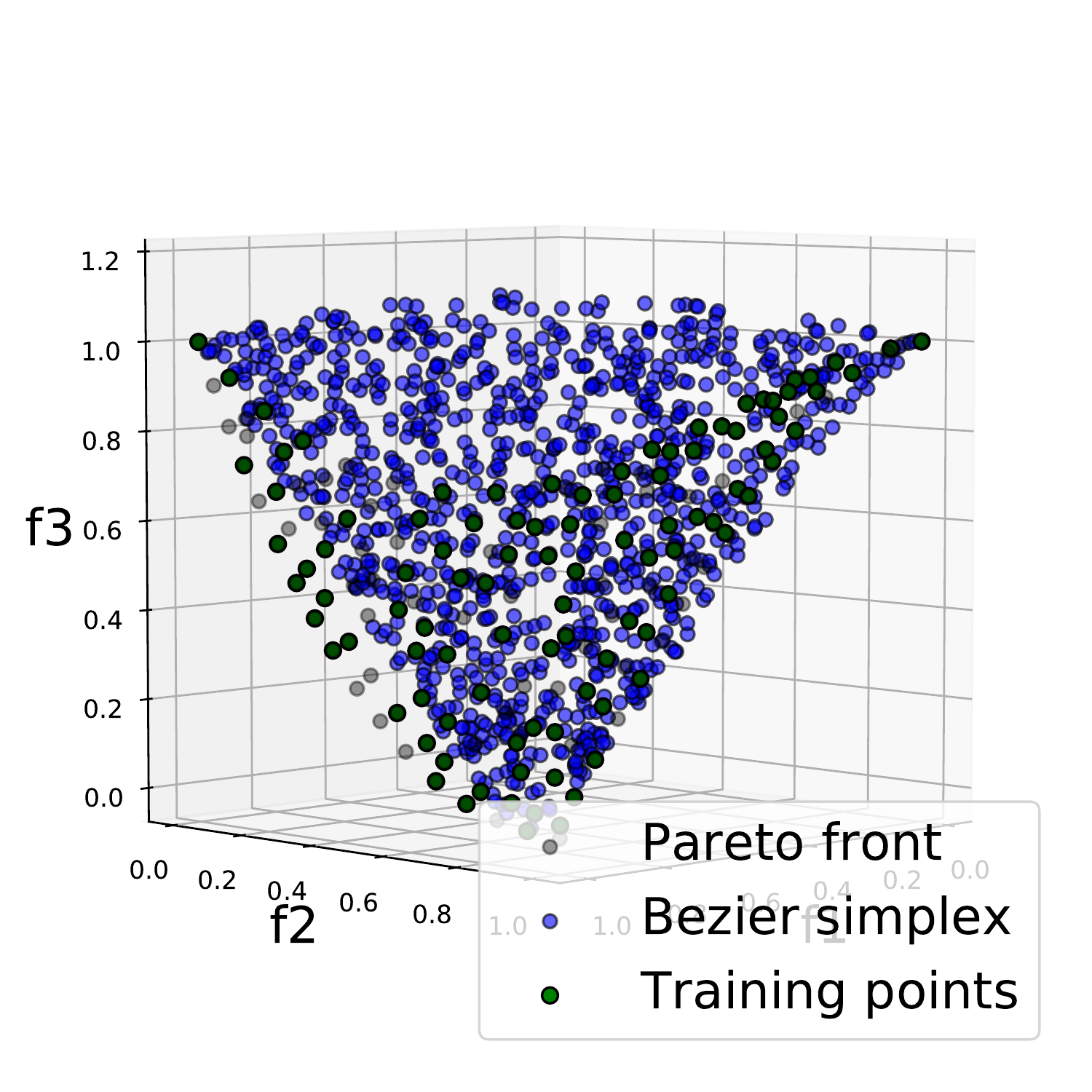}\label{fig:sigma0_aat}}
    \caption{B\'ezier triangles for 3-MED with $\sigma=0$.}
    \label{fig:sigma0}
\end{figure}
\begin{figure}[ht]
    \centering
    \subfloat[WABC]{\includegraphics[width=0.5\hsize]{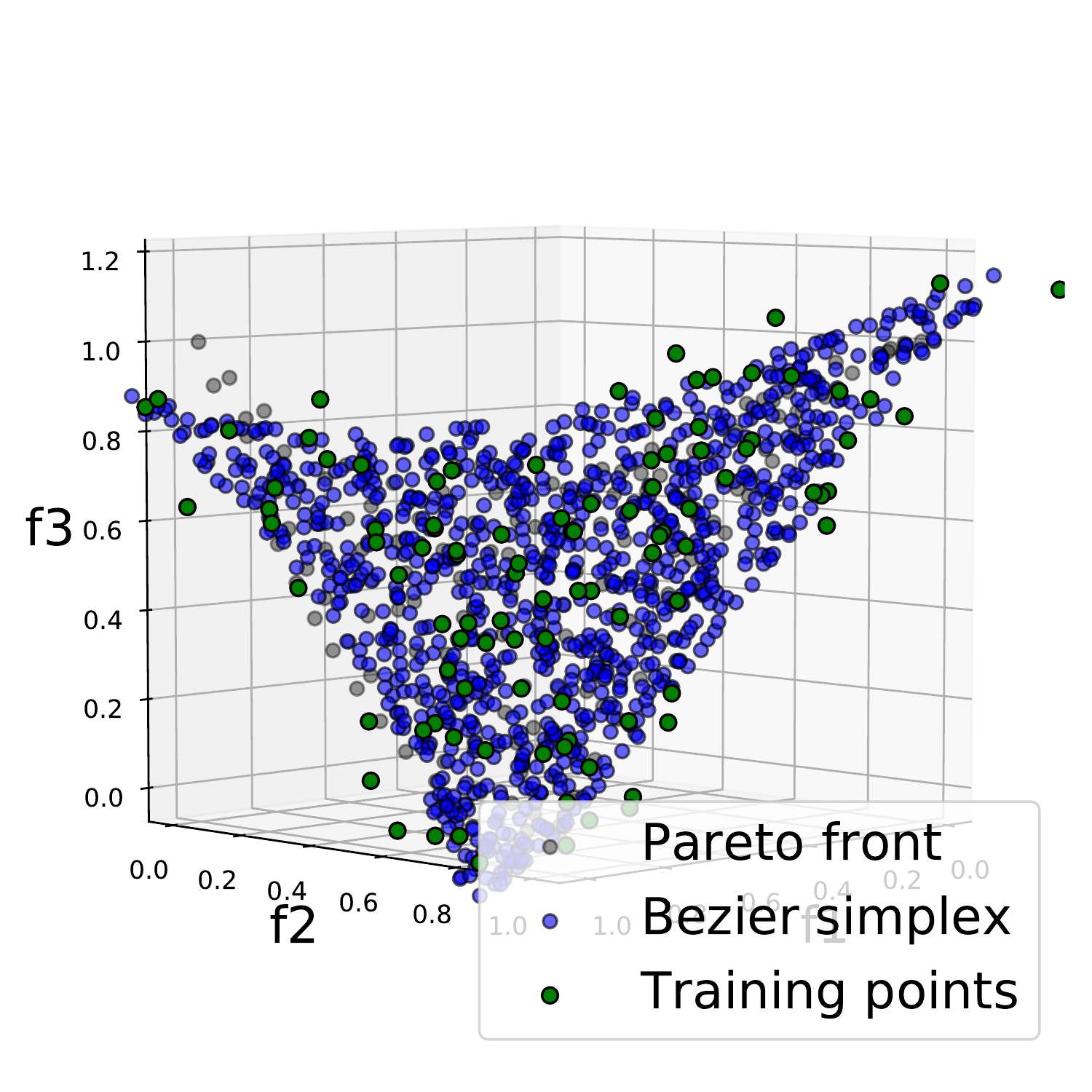}\label{fig:s0.1_abc}}
    \subfloat[All-at-once]{\includegraphics[width=0.5\hsize]{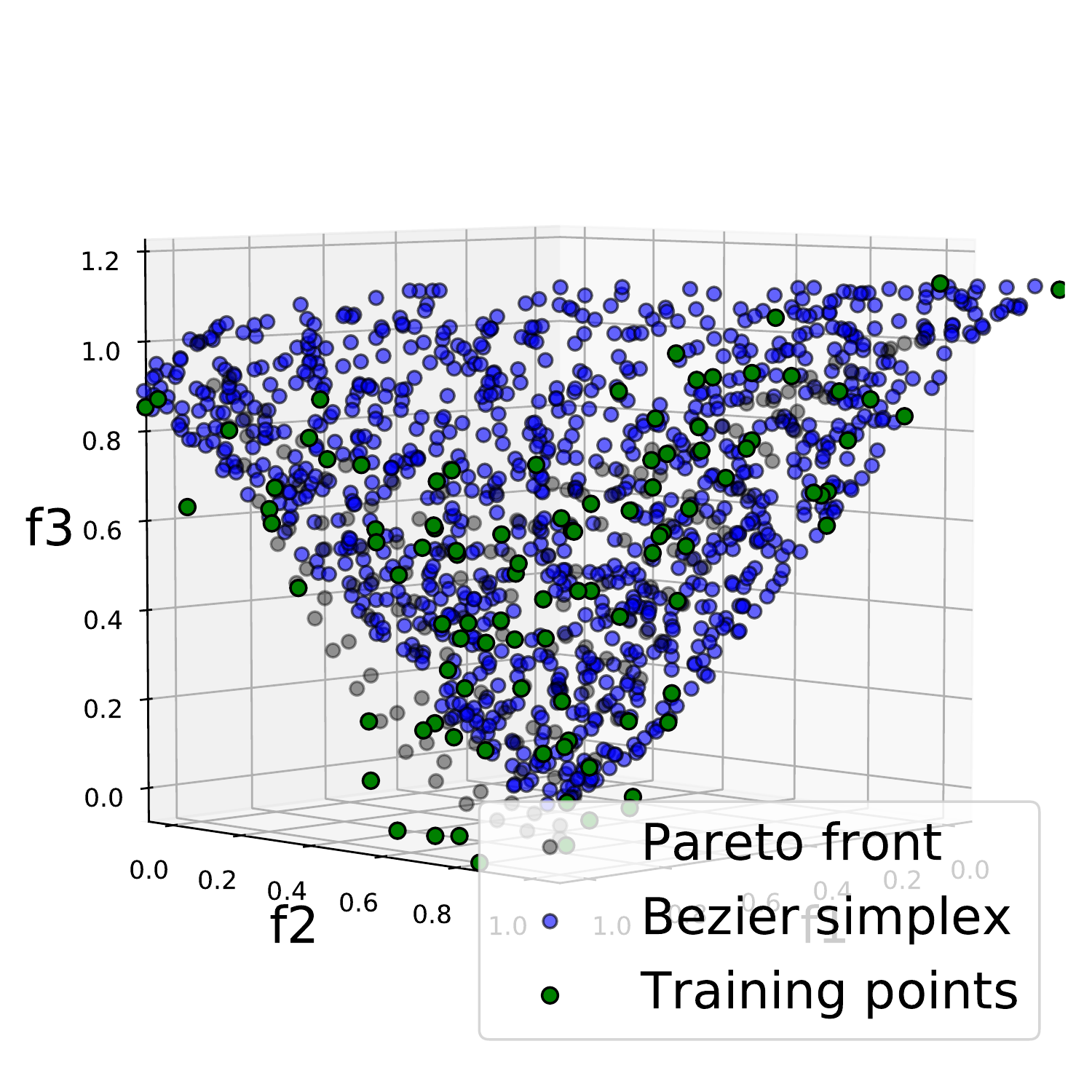}\label{fig:s0.1_aat}}
    \caption{B\'ezier triangles for 3-MED with $\sigma=0.1$.}
    \label{fig:sigma0.1}
\end{figure}

\subparagraph{Computation time}
For computation time, we have to point out that WABC was certainly slower than the all-at-once method especially when $M$ was large. 
As we have shown in the previous section, the acceptance rate decreases in WABC if $M$ increases. 
In addition, WABC evaluates the Wasserstein distance, which requires to solve a linear optimization problem. 
For these reasons, WABC is expected to be slower than the all-at-once method.

In the context of practical use cases of multi-objective optimization, however, we believe that this issue does not affect the practicality of WABC so much and can be alleviated. 
First, for multi-objective optimization problems in the real world, the number of objective functions $M$ is usually less than 10 and at most 15~(\cite{Li2015,vonLcken2014}).
Secondly, many real-world problems involve expensive simulations and/or experiments to evaluate solutions and the available sample size is often limited. 
In these situations, the computational time of WABC is not expected to be so large and it is more favourable than the all-at-one method considering its performance. 
Furthermore, the calculation of sampling in our WABC can be easily parallelized and we believe there is still room for accelerating the entire computation.

\section{Conclusion}
In this paper, we have extended the deterministic B\'ezier simplex fitting algorithm and proposed an approximate Bayesian computation (ABC) for B\'ezier simplex fitting. 
Also, we have analysed the Bias of our ABC posterior and its acceptance rate when we employ the Wasserstein distance as a metric. 
The theoretical results have been verified via numerical experiments with synthetic data. 
In addition, we have demonstrated the effectiveness and the robustness of our proposed algorithm for the datasets from multi-objective optimization problems in practical settings. 
As for future work, we point out the issue of computation time of our ABC.
Currently, our experiments have shown that our ABC is slower than the existing deterministic method. 
Thus, it is required to improve the efficiency of our proposed method by parallelizing sampling or accelerating the calculation of the Wasserstein distance by using approximation. 
Also, it would be interesting to extend our method so that it adjusts control points in an inductive manner as did in the inductive skeleton fitting \citep{Kobayashi2019}.
\bibliography{bibliography}

\clearpage
\clearpage
\appendix

\section{Proof}

\subsection{Definitions}

We generalize the definition of the Euclidean distance \eqref{euc} to aligned vectors.

\begin{definition}[Euclidean distance of aligned vector]
\label{def:euc}
Given $\bx_{1:n}$ and $\by_{1:n}$, we generalize the notion of Euclidean distance as
\begin{align*}
    d_\mathrm{E}(\bx_{1:n}, \by_{1:n})
    := 
    \sqrt{
        \sum_{i=1}^n 
        d_\mathrm{E}(\bx_i, \by_i)^2
        }.
\end{align*}
It is equivalent to the Euclidean distance in vector space where $\bx_{1:n}$ is defined.
\end{definition}

Next, we define the notation of the Balls defined by Wasserstein distance.

\begin{definition}[2-Wasserstein Ball]
\label{def:wball}
\begin{align*}
    WB_r^{nM}
    :=
    \Set{
        \by_{1:n} \in \R^{nM}
        | 
        d_\mathrm{W} (\bx_{1:n}, \by_{1:n}) \leq r
    }
\end{align*}
where $d_\mathrm{W}(\cdot, \cdot)$ is the Wasserstein-2 distance defined in \eqref{Wass2}.
\end{definition}

\begin{definition}[Permutation of $\bx_{1:n}$]
\label{def:perm}
Given aligned vector $\bx_{1:n}$ and permutation $\sigma \in S_n$, we define a new aligned vector  
\begin{align*}
    \bx_{\sigma(1:n)}
    =
    [\bx_{\sigma(1)}:\bx_{\sigma(2)}:\cdots:\bx_{\sigma(n)}],
\end{align*}
where the ":" means concatenation of vectors defined in the same manner in \Cref{def:ali}.
\end{definition}

\subsection{Lemmas}

To show our main theorems, we prove three lemmas first.
The first lemma is as follows.

\begin{lemma}
\label{lem:1}
Assume that for any $ i, j, \ \bx_i \neq \bx_j$,then there exists a constant $c>$ such that for any $\delta <c$  the Ball defined by Wasserstein-2 distance \eqref{Wass2} decomposes to a direct sum of the Balls with center $\bx_{\sigma(1:n)}$, i.e.
\begin{align*}
    W\ball{\delta}{nM}(\bx_{1:n})
    &= 
    \cup_{\sigma \in S_n}
    \ball{\sqrt{n}\delta}{nM}(\bx_{\sigma(1:n)}),
\end{align*}
where $\ball{\sqrt{n}\delta}{nM}(\bx_{\sigma_1(1:n)})
\cap
\ball{\sqrt{n}\delta}{nM}(\bx_{\sigma_2(1:n)})=\phi$ for $\sigma_1\neq \sigma_2$.
\end{lemma}

\begin{proof}
First, we show that any pair of two Balls defined by two permutations do not intersect for sufficiently small $\delta>0$, i.e.
\begin{align}
    \ball{\sqrt{n}\delta}{nM}(\bx_{\sigma_1(1:n)})
    \cap
    \ball{\sqrt{n}\delta}{nM}(\bx_{\sigma_2(1:n)})=\phi
    \text{ for }
    \sigma_1\neq \sigma_2.
    \label{sub1}
\end{align}

It is sufficient to show $\sigma_1 = \mathrm{id}, \sigma_2 = \sigma \neq \mathrm{id}$, because $\exists \sigma, s.t. \sigma_2 = \sigma \circ \sigma_1$.
We claim that for any $\delta < d_\mathrm{E}(\bx_{1:n}, \bx_{\sigma (1:n)})/3\sqrt{n}$, $\ball{\sqrt{n}\delta}{nM}(\bx_{1:n})
    \cap
    \ball{\sqrt{n}\delta}{nM}(\bx_{\sigma(1:n)})=\phi$.
 Assume that there exists $\by_{1:n}\in \ball{\sqrt{n}\delta}{nM}(\bx_{1:n})
    \cap
    \ball{\sqrt{n}\delta}{nM}(\bx_{\sigma(1:n)})$.   
  Then, by the triangle inequality
\begin{align*}
    d_\mathrm{E}(\bx_{1:n}, \bx_{\sigma(1:n)})
    &\leq
    d_\mathrm{E}(\bx_{1:n}, \by_{1:n})
    +
    d_\mathrm{E}(\bx_{\sigma(1:n)}, \by_{1:n}) \\
    &\leq
    2\sqrt{n} \delta \\
    &\leq
   \frac{2}{3} d_\mathrm{E}(\bx_{1:n}, \bx_{\sigma(1:n)})
\end{align*}
This implies that $d_\mathrm{E}(\bx_{1:n}, \bx_{\sigma(1:n)})
    \leq 0$, which contradicts the fact that $ d_\mathrm{E}(\bx_{1:n}, \bx_{\sigma(1:n)})$ is not zero.

    .

Next, we show 
\begin{align}
    W\ball{\delta}{nM}(\bx_{1:n}) 
    \subset
    \cup_{\sigma \in S_n}
    \ball{\sqrt{n}\delta}{nM}(\bx_{\sigma(1:n)})
    \label{sub2}
\end{align}
Let us take $\by_{1:n} \in W\ball{\delta}{nM}(\bx_{1:n})$, then there exists $\sigma^* \in S_n$ realizing the Wasserstein distance.
It satisfies
\begin{align*}
    \frac{1}{n} d_\mathrm{E}(\bx_{\sigma^*(1:n)}, \by_{q:n})^2
    = 
    \frac{1}{n} \sum_{i=1}^n d_\mathrm{E}(\bx_{\sigma^*(i)}, \by_i)^2 \leq 
    \delta^2.
\end{align*}
It means $\by_{1:n} \in \ball{\sqrt{n}\delta}{nM} \subset \sqcup_{\sigma \in S_n}
\ball{\sqrt{n}\delta}{nM}(\bx_{\sigma(1:n)})$, and \eqref{sub2} holds.

As the final step, we show
\begin{align}
    W\ball{\delta}{nM}(\bx_{1:n}) 
    \supset
    \cup_{\sigma \in S_n}
    \ball{\sqrt{n}\delta}{nM}(\bx_{\sigma(1:n)})
    \label{sub3}
\end{align}

Let us take $\by_{1:n} \in \cup_{\sigma \in S_n}
\ball{\sqrt{n}\delta}{nM}(\bx_{\sigma(1:n)})$, then there exists $\tilde{\sigma} \in S_n$ satisfying
\begin{align*}
    \by_{1:n} \in \ball{\sqrt{n}\delta}{nM}(\bx_{\tilde{\sigma}(1:n)})
    \Leftrightarrow
    \frac{1}{n}
    d_\mathrm{E}(\bx_{\tilde{\sigma}(1:n)}, \by_{1:n})^2
    \leq 
    \delta^2.
\end{align*}
Since the equation \eqref{sub1} is satisfied for sufficiently small $\delta>0$, the following 
\begin{align*}
    \by_{1:n} \not\in \ball{\sqrt{n}\delta}{nM}(\bx_{\sigma(1:n)})
    \Leftrightarrow
    \frac{1}{n}
    d_\mathrm{E}(\bx_{\sigma(1:n)}, \by_{1:n}) > \delta^2
\end{align*}
holds for $\sigma \neq \tilde{\sigma}$.
By combining these two inequalities, the inequality
\begin{align*}
    \frac{1}{n}
    d_\mathrm{E}(\bx_{\tilde{\sigma}(1:n)}, \by_{1:n})^2
    \leq
    \delta^2
    <
    \frac{1}{n}
    d_\mathrm{E}(\bx_{\sigma(1:n)}, \by_{1:n})^2
    ,
\end{align*}
holds for $\sigma \neq \tilde{\sigma}$.
It means that $\tilde{\sigma}$ realizes the minimum distance, and gives Wasserstein distance.
So
\begin{align*}
    d_\mathrm{W}(\bx_{1:n}, \by_{1:n}) 
    = 
    \sqrt{
        \frac{1}{n}
        d_\mathrm{E}(\bx_{\tilde{\sigma}(1:n)}, \by_{1:n})^2
    }
    \leq
    \delta,
\end{align*}
meaning \eqref{sub3}.

Combining the equations \eqref{sub2} and \eqref{sub3}, it implies $
    W\ball{\delta}{nM}(\bx_{1:n}) 
    =
    \cup_{\sigma \in S_n}
    \ball{\sqrt{n}\delta}{nM}(\bx_{\sigma(1:n)})$
\end{proof}

The next lemma is essentially equivalent to Lemma 3.6. of the paper by \cite{barber2015rate}.

\begin{lemma}[\cite{barber2015rate}]
\label{lem:2}
Suppose there is a $q$-dimensional sufficient statistics $T(\bx_{1:n})=: \bs_\bx$ .
Assume that $p(\bs_\by|\theta)$ is three times continuously differentiable on $\ball{\sqrt{n}\delta}{q}(\bs_\bx) $ for  $\theta \in\supp(p_{\mathrm{prior}}(\theta)), \delta >0$, then
\begin{align*}
    &\int_{\bs_\by \in \ball{\sqrt{n}\delta}{q}(\bs_\bx)}
        p(\bs_\by|\theta) 
        d\bs_\by
        =
    \vol{\sqrt{n}\delta}{q}
    \Big(
        p(\bs_\bx|\theta) 
        +
        \frac{(\sqrt{n}\delta)^2}
             {2q(q+2)}
        \Delta
        p(\bs_\bx | \theta)
        +
        \mathcal{O}(\delta^3)
    \Big),
\end{align*}
where $\Delta$ is Laplacian operator acting on $\bs_\bx$.
\end{lemma}

\begin{proof}
It can be proved by Taylor expansion.
Now, let us expand the integrand around $\bs_\bx$, then we get the following expression,
\begin{align*}
    p(\bs_\by|\theta)
    &=
    p(\bs_\bx|\theta)
    +
    \nabla p(\bs_\bx|\theta) 
    \cdot
    (\bs_\by - \bs_\bx)
    \\ & \quad
    +
    \frac{1}{2}
    (\bs_\by - \bs_\bx)^\top 
    H(\bs_x)
    (\bs_\by - \bs_\bx)
    + 
    r_3 (\bs_\by - \bs_\bx)^3
\end{align*}
The integral of 2nd term vanishes because $\int_{\ball{\sqrt{n}\delta}{q}(\bs_\bx)} (\bs_\by - \bs_\bx) d\bs_\by = 0$ by symmetry.
To calculate the 3rd term, we diagonalize the Hessian matrix $H(\bs_\bx)$.
This is done without any Jacobian because the $H(\bs_\bx)$ is a symmetric matrix, and can be diagonalized by orthogonal matrix.
In such basis, the integral reduces to the following
\begin{align*}
    \frac{1}{2}
    \int_{\ball{\sqrt{n}\delta}{q}({\bf 0})}                 \sum_{\mu=1}^q 
        \lambda_\mu(\bs_\bx) 
        u_\mu^2
        d{\bf u}
    &=
    \frac{1}{2}
    \sum_{\mu=1}^q \lambda_\mu(\bs_\bx)
    \int_{\ball{\sqrt{n}\delta}{q}({\bf 0})} 
        \frac{{\bf u}^2}
             {q}
        d{\bf u} \\
    &=
    \frac{1}{2}
    \Delta p(\bs_\bx|\theta)
    \frac{\vol{\sqrt{n}\delta}{q}}{q}
    \frac{(\sqrt{n}\delta)^{2}}{q+2}.
\end{align*}
The integral of the error term $r_3$ is scaled smaller at least than $\mathcal{O}(\delta^{3})$ by the following discussion. Since the third derivative of $p$, is continuous, it has a maximum value $M$ on $\ball{\sqrt{n}\delta}{q}$. Also, $\bs_\by - \bs_\bx$ is always bounded by $\delta$ , which means that the third term is bounded by $M\delta^3$.
\end{proof}

Combining \Cref{lem:1} and \Cref{lem:2}, we acquire the following lemma, which is the most important one to show our main theorems.

\begin{lemma}
\label{lem:3}

Assume that $p(\bx_{1:n}|\theta)$ is three times continuously differentiable on $\ball{\sqrt{n}\delta}{q}(\bx_{1:n}) $ for  $\theta \in\supp(p_{\mathrm{prior}}(\theta))$ and $\delta >0$, then
\begin{align*}
    &\int_{d_\mathrm{W}(\bx_{1:n}, \by_{1:n}) \leq \delta} 
        p(\by_{1:n}|\theta) 
        d\by_{1:n} 
    = 
    \vol{\sqrt{n}\delta}{q}
    \sum_{\sigma \in S_n}
    \Big( 
        p(\bx_{\sigma(1:n)}|\theta) 
        + 
        \frac{(\sqrt{n}\delta)^2}{2q(q+2)} 
        \Delta 
        p(\bx_{\sigma(1:n)} | \theta)
                + 
        \mathcal{O}(\delta^{3})
        \Big)
\end{align*}
where $q=nM$, and $\Delta$ is Laplacian operator acting on $\bx_{\sigma(1:n)}$.
\end{lemma}
\begin{proof}
The integral region in the left hand side of the equation is the set $W\ball{\delta}{nM}(\bx_{1:n})$.
It has the decomposition provided by \Cref{lem:1},
\begin{align*}
    W\ball{\delta}{nM}(\bx_{1:n}) 
    =
    \cup_{\sigma \in S_n}
    \ball{\sqrt{n}\delta}{nM}(\bx_{\sigma(1:n)})
    ,
\end{align*}
with no intersection.
It implies the integral reduces to the sum of integrals over $\ball{\sqrt{n}\delta}{nM}(\bx_{\sigma(1:n)})$:
\begin{align*}
    \int_{d_\mathrm{W}(\bx_{1:n}, \by_{1:n}) \leq \delta} 
        p(\by_{1:n}|\theta) 
        d\by_{1:n} 
    &=
    \sum_{\sigma \in S_n} 
    \int_{\ball{\sqrt{n}\delta}{nM}(\bx_{\sigma(1:n)})}
        p(\by_{1:n}|\theta) 
        d\by_{1:n}
    .
\end{align*}
Here, we choose the summary statistics as $T=\mathrm{id}$, i.e. $\bs_\bx = \bx_{1:n}$, then, by applying \Cref{lem:2}, we can get the proof.
\end{proof}

\subsection{Theorems}
First, let us prove \Cref{th:3}.
Let us show the statement here again.
\setcounter{theorem}{2}
\begin{theorem}
Let $h(\theta)$ be a function of $\theta$ that $\E_{\mathrm{posterior}}[h(\theta)]$ is not divergent, and the likelihood is three times continuously differentiable with respect to $\bx_{1:n}$, then WABC also has order $\delta^2$ bias:
\[
    \E_\mathrm{WABC}[h(\theta)] 
    = 
    \E_{\mathrm{posterior}}[h(\theta)] 
    + C_h(\bx_{1:n}) \delta^2
    + \mathcal{O}(\delta^3),
\]
where $C_h(\bx_{1:n})$ is a value depending only on $\bx_{1:n}$ and the function $h$.
\end{theorem}

\begin{proof}
The main part of the proof is done by discussing the $\delta$ scaling of WABC prior distribution.
Let us remind the definition of WABC posterior.
It has the following form,
\begin{align}
    p_{\mathrm{WABC}}^{(\delta)}(\theta|\bx_{1:n})
    &=
    \frac{
        p_{\mathrm{prior}}(\theta) 
        \int_{d_\mathrm{W}(\bx_{1:n}, \by_{1:n}) \leq \delta} 
            p(\by_{1:n}|\theta) 
            d\by_{1:n}
        }{
        \int p_{\mathrm{prior}}(\tilde{\theta})d\tilde{\theta} 
        \int_{d_\mathrm{W}(\bx_{1:n}, \by_{1:n}) \leq \delta} 
            p(\by_{1:n}|\tilde{\theta}) 
            d\by_{1:n}
        }
    \label{sub4}
    .
\end{align}
In the final line, we can apply \Cref{lem:3} both in denominator and numerator.
As we noted, our likelihood is defined by the product of data \eqref{prod}, so it has the following symmetry
\begin{align*}
    \forall \sigma \in S_n, \quad p(\bx_{\sigma(1:n)}|\theta)
    =
    p(\bx_{1:n}|\theta),
\end{align*}
and it means 
\begin{align}
    &\sum_{\sigma \in S_n}
    \Big(
        p(\bx_{\sigma(1:n)}|\theta)
        +
        \frac{(\sqrt{n}\delta)^2}{2q(q+2)} 
        \Delta 
        p(\bx_{\sigma(1:n)} | \theta)
        +
        \mathcal{O}(\delta^3)
    \Big)
=
    n!
    \Big(
    p(\bx_{1:n}|\theta)
    +
    c(\bx_{1:n}, \theta) \delta^2
    +
    \mathcal{O}(\delta^3)
    \Big)
    \label{sub5}
    ,
\end{align}
where $c$ is defined as
\[
    c(\bx_{1:n}, \theta)
    =
    \frac{\sum_{\sigma \in S_n}
    \Delta 
    p(\bx_{\sigma(1:n)} | \theta)}
    {2q(q+2)(n-1)!} 
    .
\]
Substituting this scaling law to the WABC posterior expression \eqref{sub4}, we get
\begin{align*}
    p_{\mathrm{WABC}}^{(\delta)}(\theta|\bx_{1:n})
    &=
    \frac{
        p_{\mathrm{prior}}(\theta) 
        \Big(
            p(\bx_{1:n}|\theta)
            +
            c(\bx_{1:n}, \theta) \delta^2
            +
            \mathcal{O}(\delta^3)
        \Big)
        }{
        \int p_{\mathrm{prior}}(\tilde{\theta})d\tilde{\theta} 
        \Big(
            p(\bx_{1:n}|\tilde{\theta})
            +
            c(\bx_{1:n}, \tilde{\theta}) \delta^2
            +
            \mathcal{O}(\delta^3)
        \Big)
        }
    \\
    &=
    \underbrace{
    \frac{
        p_{\mathrm{prior}}(\theta) 
        p(\bx_{1:n}|\theta)
        }{
        \int p_{\mathrm{prior}}(\tilde{\theta})d\tilde{\theta} 
        p(\bx_{1:n}|\tilde{\theta})
        }
    }_{p_{\mathrm{posterior}}(\theta|\bx_{1:n})}
    +
    p_{\mathrm{prior}}(\theta) 
    C(\bx_{1:n}, \theta) \delta^2
    +
    \mathcal{O}(\delta^3)
    ,
\end{align*}
where we also define
\[
    C(\bx_{1:n}, \theta)
    =
    \frac{1}{p(\bx_{1:n})}
    \Big(
    c(\bx_{1:n}, \theta)
    -
    \frac{\int 
            p_{\mathrm{prior}}(\tilde{\theta})
            c(\bx_{1:n}, \tilde{\theta})
            d\tilde{\theta} 
        }{p(\bx_{1:n})}
    \Big)
\]
Then, we can get immediately the \Cref{th:3} by considering $\E_{\theta\sim p_\mathrm{ABC}^{(\delta)}}[h(\theta)] = \int h(\theta) p_\mathrm{ABC}^{(\delta)}(\theta|\bx_{1:n}) d\theta$, and define $C_h(\bx_{1:n}) = \E_{\theta\sim p_{\mathrm{prior}}}[h(\theta)C(\bx_{1:n}, \theta)]$
\end{proof}

Next, we prove \Cref{th:4}.

\begin{theorem}
If the likelihood is three times continuously differentiable with respect to $\bx_{1:n}$, then the acceptance probability of WABC is 
\[
    p_\mathrm{accept} = (n!) p(\bx_{1:n})
    \vol{\sqrt{n}\delta}{q} 
    ( 1 + o(1) )
\]
where $q=nM$.
We need to run $N_\mathrm{ABC}/p_\mathrm{accept}$ accept/reject trials during \Cref{alg0} in average for gathering $N_\mathrm{ABC}$ samples.
\end{theorem}

\begin{proof}
The acceptance probability is
\begin{align*}
    p_\mathrm{accept} 
    &= 
    \int
        p_{\mathrm{prior}}(\theta)
        d\theta
            \int_{d_\mathrm{W}(\bx_{1:n}, \by_{1:n}) \leq \delta}
                p(\by_{1:n}|\theta)
                d\by_{1:n}
    ,
\end{align*}
then, we can apply \Cref{lem:3}, and equation \eqref{sub5}.
The result is
\begin{align*}
    p_\mathrm{accept} 
    &= 
    \int
        p_{\mathrm{prior}}(\theta)
        d\theta
            \vol{\sqrt{n}\delta}{q}
            n!
            \Big(
            p(\bx_{1:n}|\theta)
            +
            c(\bx_{1:n}, \theta) \delta^2
            +
            \mathcal{O}(\delta^3)
            \Big)
    \\
    &=
    \vol{\sqrt{n}\delta}{q}
    n!
    \Big(
    p(\bx_{1:n})
    +
    \int
        p_{\mathrm{prior}}(\theta)
        c(\bx_{1:n}, \theta) \delta^2
        d\theta
        +
            \mathcal{O}(\delta^3)
    \Big)
    ,
\end{align*}
then the first term gives the dominant contribution, and the remaining terms are negligible compared to the first term.
It completes the proof.
\end{proof}

\section{Problem Definition}
\paragraph{Schaffer}
is a one-variable two-objective problem defined by:
\begin{align*}
\text{minimize }
&f_1(x) = x^2,\\
&f_2(x) = {(x-2)}^2\\
\text{subject to }&-100000 \le x \le 100000.
\end{align*}

\paragraph{Viennet2}
is a two-variable three-objective problem defined by:
\begin{align*}
\text{minimize }
&f_1(x) = \frac{{(x_1-2)}^2}{2} + \frac{{(x_2+1)}^2}{13}+3,\\
&f_2(x) = \frac{{(x_1+x_2-3)}^2}{36} + \frac{{(-x_1+x_2+2)}^2}{8}-17,\\
&f_3(x) = \frac{{(x_1+2x_2-1)}^2}{175} + \frac{{(2x_2-x_1)}^2}{17}-13\\
\text{subject to }&-4\le x_1,x_2\le4.
\end{align*}

\paragraph{$M$-MED}
is an $M$-variable $M$-objective problem defined by:
\begin{align*}
\text{minimize }
&f_m (x) = \left(\frac{1}{\sqrt{2}}\|x - e_m\|\right)^{p_m}~&&(m=1,\ldots,M)\\
\text{subject to }
&-5.12 \le x_i\le 5.12~&&(m=1,\ldots,M)\\
\text{where }
&p_m = \exp\left(\frac{2(m-1)}{M-1} - 1\right)&&(m=1,\ldots,M),\\
& e_m = (0,\dots,0,\underbrace{1}_{m\text{-th}},0,\dots,0)&&(m=1,\ldots,M).
\end{align*}

\section{Data Description}
The description of data used in the numerical experiments is summarized in \Cref{tbl:datasets}:
\Cref{tbl:datasets} shows the sample size of each dataset.

\begin{table}[ht]
\centering
\small
\caption{Sample size on the Pareto front for each dataset.}\label{tbl:datasets}
\begin{tabular}{lrr}
\toprule
\textbf{Problem}&$M$&sample size\\
\midrule
Schaffer&2&201\\
3-MED&3&153\\
Viennet2&3&8122\\
5-MED&5&4845\\
\bottomrule
\end{tabular}
\end{table}

\section{Additional Experimental Results}
For each problem and method, the average and the standard deviation of the computation time, GD and IGD when $n=50$ and $150$ with $\sigma\in \{0, 0.05, 0.1\}$ are shown in \Cref{tab:result_N50} and \Cref{tab:result_N150}, respectively.

\begin{table*}[ht]
\footnotesize
    \centering
    \caption{GD and IGD (avg.$\pm$s.d. over 20 trials) with $n=50$. The best scores with signifiance level $p<0.05$ are shown in bold.}
    \label{tab:result_N50}
    \begin{tabular}{lrrccrcc}
    \toprule
        Problem &$\sigma$ & &WABC &&&All-at-once &  \\
                &         & Time & GD &IGD& Time & GD &IGD\\ \midrule
Schaffer	&0	&37.1	&1.47E-02$\pm$7.00E-03	&1.65E-02$\pm$7.60E-03	&1.3	&\textbf{2.49E-03$\pm$4.78E-05}	&\textbf{2.09E-03$\pm$1.38E-03}\\
($M=2$)		&0.05	&39.9	&1.80E-02$\pm$6.89E-03	&1.94E-02$\pm$6.94E-03	&1.6	&1.47E-02$\pm$5.71E-03	&2.04E-02$\pm$9.32E-03\\ 
		&0.1	&351.7	&\textbf{2.39E-02$\pm$9.10E-03}	&\textbf{2.44E-02$\pm$1.02E-02}	&1.9	&3.15E-02$\pm$1.35E-02	&3.52E-02$\pm$1.48E-02\\ \hline
3-MED	&0	&105.0	&\textbf{5.42E-02$\pm$3.10E-03}	&\textbf{4.14E-02$\pm$4.37E-03}	&2.6	&9.03E-02$\pm$1.48E-02	&4.95E-02$\pm$9.51E-03\\
($M=3$)		&0.05	&191.4	&\textbf{6.19E-02$\pm$5.49E-03}	&\textbf{4.84E-02$\pm$5.52E-03}	&2.1	&1.02E-01$\pm$2.04E-02	&5.99E-02$\pm$1.26E-02\\
		&0.1	&904.3	&\textbf{8.95E-02$\pm$9.22E-03}	&\textbf{6.83E-02$\pm$9.54E-03}	&3.7	&1.38E-01$\pm$3.36E-02	&8.89E-02$\pm$2.29E-02\\\hline
Viennet2	&0	&148.2	&\textbf{2.16E-02$\pm$5.51E-03}	&\textbf{3.20E-02$\pm$7.35E-03}	&4.3	&1.22E+01$\pm$1.62E+01	&9.05E-02$\pm$2.77E-02\\
($M=3$)		&0.05	&254.2	&\textbf{4.50E-02$\pm$5.97E-03}	&\textbf{3.88E-02$\pm$7.59E-03}	&3.3	&1.18E-01$\pm$8.61E-02	&8.32E-02$\pm$2.20E-02\\
		&0.1	&1114.0	&\textbf{1.09E-01$\pm$9.81E-03}	&\textbf{6.37E-02$\pm$9.84E-03}	&5.5	&1.37E-01$\pm$6.21E-02	&1.19E-01$\pm$3.20E-02\\ \hline
5-MED	&0	&3990.0	&\textbf{8.70E-02$\pm$8.51E-03}	&\textbf{1.41E-01$\pm$1.57E-02}	&26.2	&1.39E-01$\pm$3.18E-02	&2.16E-01$\pm$2.27E-02\\
($M=5$)		&0.05	&2989.6	&\textbf{9.93E-02$\pm$6.84E-03}	&\textbf{1.52E-01$\pm$2.07E-02}	&45.6	&3.95E-01$\pm$6.26E-01	&4.15E-01$\pm$7.33E-01\\
		&0.1	&1851.5	&\textbf{1.34E-01$\pm$1.21E-02}	&\textbf{1.83E-01$\pm$2.55E-02}	&214.3	&1.36E+00$\pm$1.33E+00	&1.51E+00$\pm$1.55E+00\\
\bottomrule
    \end{tabular}
\end{table*}

\begin{table*}[ht]
\footnotesize
    \centering
    \caption{GD and IGD (avg.$\pm$s.d. over 20 trials) with $n=150$. The best scores with signifiance level $p<0.05$ are shown in bold.}
    \label{tab:result_N150}
    \begin{tabular}{lrrccrcc}
    \toprule
        Problem &$\sigma$ & &WABC &&&All-at-once &  \\
                &         & Time & GD &IGD& Time & GD &IGD\\ \midrule
Schaffer	&0	&119.3	&6.44E-03$\pm$1.54E-03	&5.82E-03$\pm$1.60E-03	&2.8	&\textbf{2.48E-03$\pm$6.28E-05}	&\textbf{1.08E-03$\pm$1.20E-04}\\
($M=2$)		&0.05	&651.8	&1.02E-02$\pm$2.69E-03	&9.24E-03$\pm$2.74E-03	&3.2	&1.01E-02$\pm$3.72E-03	&1.36E-02$\pm$6.83E-03\\
		&0.1	&579.2	&\textbf{1.60E-02$\pm$6.10E-03}	&\textbf{1.47E-02$\pm$6.35E-03}	&3.1	&2.37E-02$\pm$1.10E-02	&1.76E-02$\pm$7.26E-03\\ \hline
3-MED	&0	&1152.9	&\textbf{5.01E-02$\pm$2.10E-03}	&\textbf{3.30E-02$\pm$2.45E-03}	&4.1	&1.11E-01$\pm$3.97E-03	&4.21E-02$\pm$1.29E-03\\
($M=3$)		&0.05	&1023.8	&\textbf{5.49E-02$\pm$2.66E-03}	&\textbf{3.67E-02$\pm$3.15E-03}	&4.1	&1.09E-01$\pm$1.44E-02	&4.68E-02$\pm$6.99E-03\\
		&0.1	&888.0	&\textbf{7.32E-02$\pm$3.90E-03}	&\textbf{4.87E-02$\pm$3.80E-03}	&4.6	&1.25E-01$\pm$1.70E-02	&6.09E-02$\pm$1.12E-02\\\hline
Viennet2	&0	&655.9	&\textbf{2.08E-02$\pm$3.93E-03}	&\textbf{2.76E-02$\pm$3.65E-03}	&13.6	&4.47E+00$\pm$8.77E+00	&9.67E-02$\pm$8.41E-02\\
($M=3$)	&0.05	&1737.9	&\textbf{4.17E-02$\pm$5.16E-03}	&\textbf{3.29E-02$\pm$4.11E-03}	&4.7	&9.24E-02$\pm$4.74E-02	&5.98E-02$\pm$1.72E-02\\
		&0.1	&1168.9	&\textbf{8.31E-02$\pm$1.26E-02}	&\textbf{5.69E-02$\pm$1.04E-02}	&10.4	&1.26E-01$\pm$5.59E-02	&1.35E-01$\pm$1.03E-01\\ \hline
5-MED	&0	&2291.9	&\textbf{9.61E-02$\pm$1.10E-02}	&\textbf{1.33E-01$\pm$1.03E-02}	&33.9	&1.68E-01$\pm$2.79E-02	&1.93E-01$\pm$1.25E-02\\
($M=5$)		&0.05	&2099.8	&\textbf{1.01E-01$\pm$8.51E-03}	&\textbf{1.38E-01$\pm$1.00E-02}	&36.9	&1.84E-01$\pm$4.82E-02	&1.92E-01$\pm$2.35E-02\\
		&0.1	&1636.4	&\textbf{1.33E-01$\pm$1.47E-02}	&\textbf{1.64E-01$\pm$1.51E-02}	&226.4	&1.09E+00$\pm$1.24E+00	&1.21E+00$\pm$1.48E+00\\
\bottomrule
    \end{tabular}
\end{table*}

\end{document}